\documentclass[compsoc]{IEEEtran}

\title{Analytic Properties of Trackable Weak Models}

\author{Mark Chilenski, George Cybenko~\IEEEmembership{Fellow,~IEEE,} Isaac Dekine, Piyush Kumar, and Gil Raz\IEEEcompsocitemizethanks{\IEEEcompsocthanksitem M.~Chilenski, I.~Dekine, P.~Kumar, and G.~Raz are with Systems \& Technology Research. \IEEEcompsocthanksitem G.~Cybenko is with Dartmouth College. \IEEEcompsocthanksitem Preprint of paper to be submitted to \emph{IEEE Transactions on Network Science and Engineering.} \IEEEcompsocthanksitem Copyright \textcopyright\ 2020 Systems \& Technology Research LLC}}

\usepackage{xcolor}
\usepackage{amsmath}
\usepackage{amsfonts}
\usepackage{graphicx}
\usepackage{subfig}
\usepackage{xspace}
\usepackage{booktabs}
\usepackage{caption}
\usepackage{amsthm}
\usepackage{array}
\usepackage{siunitx}

\newcommand{\bluecircle}{\raisebox{-1pt}{\includegraphics{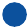}}\xspace}
\newcommand{\redsquare}{\raisebox{-1pt}{\includegraphics{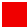}}\xspace}
\newcommand{\greendiamond}{\raisebox{-2.1pt}{\includegraphics{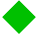}}\xspace}
\newcommand{\orangetriangle}{\raisebox{-1pt}{\includegraphics{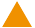}}\xspace}

\newcommand{\purplerect}{\raisebox{-1pt}{\includegraphics{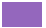}}\xspace}

\newcommand{\Aut}{\mathrm{Aut}}

\graphicspath{{./figures/}}

\newcolumntype{C}{>{$}c<{$}}

\newtheorem{theorem}{Theorem}
\newtheorem{corollary}[theorem]{Corollary}

\begin{document}
\maketitle

\begin{abstract}
	We present several new results on the feasibility of inferring the hidden states in strongly-connected trackable weak models. Here, a weak model is a directed graph in which each node is assigned a set of colors which may be emitted when that node is visited.  A hypothesis is a node sequence which is consistent with a given color sequence.  A weak model is said to be trackable if the worst case number of such hypotheses grows as a polynomial in the sequence length. We show that the number of hypotheses in strongly-connected trackable models is bounded by a constant and give an expression for this constant. We also consider the problem of reconstructing which branch was taken at a node with same-colored out-neighbors, and show that it is always eventually possible to identify which branch was taken if the model is strongly connected and trackable. We illustrate these properties by assigning transition probabilities and employing standard tools for analyzing Markov chains. In addition, we present new results for the entropy rates of weak models according to whether they are trackable or not. These theorems indicate that the combination of trackability and strong connectivity dramatically simplifies the task of reconstructing which nodes were visited. This work has implications for any problem which can be described in terms of an agent traversing a colored graph, such as the reconstruction of hidden states in a hidden Markov model (HMM).
\end{abstract}


\section{Introduction}
\IEEEPARstart{C}{onsider} an agent traversing a directed graph whose nodes, upon being visited, emit a color taken from a set of possibilities.
We are interested in identifying the nodes visited by the agent based on the colors observed.
If there are probabilities associated with the edges and color emissions then the system described is a \emph{hidden Markov model (HMM)} \cite{RabinerIEEE1989}, otherwise it is a \emph{weak model} \cite{CrespiACM2008}.
Weak models are used in a variety of applications including target tracking in sensor networks \cite{CrespiACM2008} and cybersecurity \cite{ChilenskiSPIE2018}.

This paper concerns itself only with the reconstruction of the node sequence given the observed color sequence.
We assume that the graph and its coloring are known \emph{a priori}, as we are not studying the more challenging task of inferring the graph and its coloring from observed sequences of colors \cite{kearns1994cryptographic}.

In general, no guarantees can be made about the ability to reconstruct the sequence of nodes visited.
In our previous work, we laid out a taxonomy of observability classes for which different guarantees can be made \cite{ChilenskiTaxon2018}.
In particular, \emph{trackable} models are weak models for which the number of node sequences (the ``hypotheses'') consistent with an observed color sequence is bounded by a polynomial in the sequence length \cite{CrespiACM2008}.
For models which are not trackable, the growth rate is exponential.
While polynomial growth is certainly better than exponential growth, a trackable model can still give rise to an unbounded number of hypotheses, which suggests that trackability is not a sufficient condition for accurate reconstruction of the node sequence.

In this paper, however, we present several new results which show that the node sequence reconstruction problem for trackable models which are strongly connected is much more feasible than in the more general case with transient nodes.
As strong connectivity is a reasonable assumption when observing the steady state behavior of a system, these results suggest that the combination of trackability with strong connectivity can be sufficient to enable accurate reconstruction of the node sequence.
This has important implications for the design of sensor networks and tracking systems, as trackability imposes much weaker constraints on the model's coloring than the observability classes with stronger guarantees such as unifilar \cite{ShengIEEE2005} and observable \cite{JungersDAM2011}.

The structure of this paper is as follows: after establishing the necessary background and notation in Section~\ref{sec:bg}, we show in Section~\ref{sec:bounded} that the number of hypotheses for strongly-connected trackable models is bounded by a constant.
Then in Section~\ref{sec:recon} we show that, for such models, it is always eventually possible to reconstruct which branch was taken at a point where there are multiple out-neighbors capable of emitting the same color.
We illustrate how to quantify what ``eventually'' means using the mean recurrence time and find that the results agree well with a simulation.
Section~\ref{sec:maxSize} provides an upper bound on the size of the hypothesis set for strongly-connected trackable models.
Section~\ref{sec:entropy} presents new results about the entropy rates of possible node sequences given observation sequences.
Section~\ref{sec:conc} summarizes these results.

\section{Background: Weak Models and Trackability}
\label{sec:bg}
A \emph{weak model} $G = (V, E, L, \Phi)$ consists of a set of nodes $V$ (the states a system can transition between), a set of edges $E$ consisting of ordered pairs of nodes (the allowed state transitions), a set of possible colors $\Phi$ (called the ``symbols'' in other work on weak models) which can be emitted, and a mapping $L:V\to2^\Phi$ which indicates which subset of $\Phi$ can be emitted by a given node \cite{CrespiACM2008}.
This can be seen as a \emph{node-multi-colored directed graph}, and is a non-probabilistic version of a \emph{hidden Markov model} (HMM) with discrete symbols \cite{RabinerIEEE1989}.
We restrict our attention to graphs with a finite number of nodes.

Any multi-colored model can be transformed to an equivalent single-colored model (i.e., a model for which $|L(v)|=1$ $\forall v\in V$) by replacing each multi-colored node with multiple nodes, one for each color, as shown in Figure~\ref{fig:mcTrans}.
\begin{figure}
	\centering
	\subfloat[Multi-colored]{\includegraphics[scale=0.15]{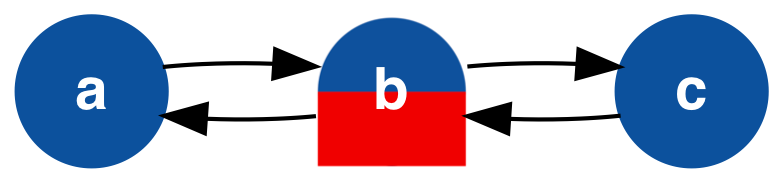}}
	\subfloat[Single-colored]{\includegraphics[scale=0.15]{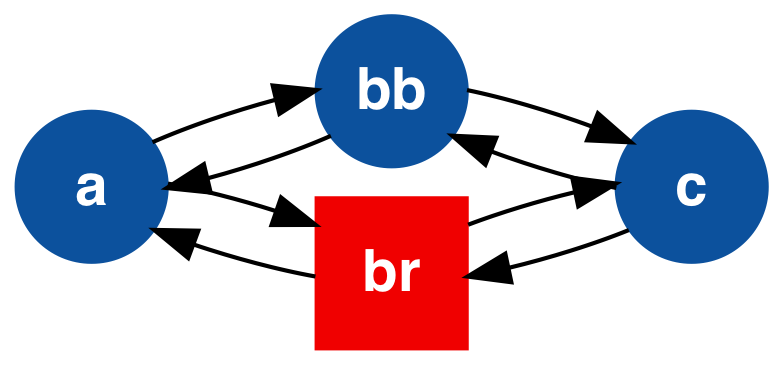}}
	\caption{Transformation of a multi-colored model into a single-colored model. Node $b$ can emit either \protect\bluecircle or \protect\redsquare, so it is split into nodes $bb$ and $br$, respectively. (This paper will use both color and shape to distinguish the ``colors'' which can be emitted by a given node.)}
	\label{fig:mcTrans}
\end{figure}
We will therefore assume, without loss of generality, that all models in this paper are single-colored.
In this case, we can replace $L$ with the function $\ell: V\to \Phi$, which is equivalent to the \emph{lumping function} in \cite{GurvitsLAA2005,geiger2014lumpings}.
An HMM with single-colored nodes is also known as a \emph{lumped Markov chain}.

Consider the problem of tracking a system described by such a model.
The structure of the graph is known, as is the sequence of colors emitted as the system transitions between the various states.
The goal is to reconstruct the sequence of nodes visited.
In general, no guarantees can be made about the accuracy of this inference problem.
As mentioned above, however, there is a specific class of weak models called \emph{trackable} weak models for which the number of hypotheses consistent with an observed color sequence is bounded by a polynomial in the sequence length \cite{CrespiACM2008}.
Formally, let $\mathcal{H}_G(Y_{[t]})$ be the set of node sequences which are consistent with the color sequence $Y_{[t]}=(Y_1, \dots, Y_t)$.
We refer to any such node sequence as a \emph{hypothesis}.
Let $n_G(t)=\max_{Y_{[t]}}|\mathcal{H}_G(Y_{[t]})|$ be the worst-case number of hypotheses for a color sequence of length $t$.
A trackable model is defined to be a model for which $n_G(t)=O(t^k)$ for some $k\geq0$.
Models which are not trackable must have exponential hypothesis growth rates: $n_G(t)=\Omega(2^{ct})$ for some $c>0$.

The original work on trackable models characterizes trackability both in terms of the joint spectral radius of the transition matrices defining the model as well as by considering the node sequences which are consistent with all possible color sequences.
As we describe in another publication, trackable models are also characterized by the absence of \emph{intersecting} cycles which permit the same color sequence \cite{ChilenskiTaxon2018}.
Two cycles $\pi_1, \pi_2 : \mathbb{Z}\to V$ indexed by $i$ are said to be intersecting if there is at least one $i$ such that $\pi_1(i)=\pi_2(i)$ and are said to permit the same color sequence if $\ell(\pi_1(i))= \ell(\pi_2(i))$ $\forall i$.\footnote{The equivalent definition of permitting the same color sequence for a multi-colored model is $L(\pi_1(i))\cap L(\pi_2(i))\neq\emptyset$ $\forall i$.}

Examples of trackable and untrackable models are given in Figure~\ref{fig:examples}.
\begin{figure}
	\centering
	\subfloat[Untrackable]{\includegraphics[scale=0.15]{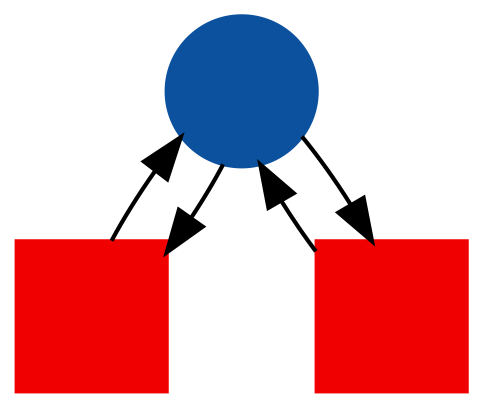}\label{sf:untrack}}
	\quad
	\subfloat[Trackable]{\includegraphics[scale=0.15]{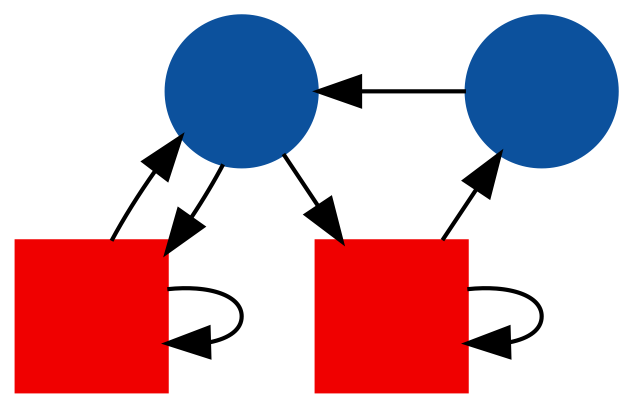}\label{sf:track}}
	\caption{Basic examples of models which are \protect\subref{sf:untrack} untrackable and \protect\subref{sf:track} trackable.}
	\label{fig:examples}
\end{figure}
In the model shown in Figure~\ref{sf:untrack}, the two intersecting cycles of the form (\bluecircle, \redsquare, \bluecircle) cause the model to be untrackable, because this color sequence begins and ends at the central \bluecircle node but permits two paths -- the model lacks the ``unique path property'' defined in \cite{CrespiACM2008}.
Indeed, every time the color sequence (\bluecircle, \redsquare, \bluecircle) is observed the number of possible node sequences doubles, resulting in the exponential growth expected for an untrackable model.

The model in Figure~\ref{sf:track}, on the other hand, does not possess this pathology.
While it may not be possible to know what node the system is at for arbitrarily long periods of time (e.g., when a color sequence of the form (\redsquare, \redsquare, \dots) is observed), there is no color sequence which causes exponential growth in the number of hypotheses.
In fact, the results of Section~\ref{sec:bounded} show that the number of hypotheses is $O(1)$ for this model -- at any given time, there are only one or two possible node sequences which are consistent with the observed color sequence.

\section{Bounded Growth in Strongly-Connected Trackable Models}
\label{sec:bounded}
It has been shown that for a trackable model $n_G(t)=O(t^k)$ with $k>0$ (unbounded polynomial growth) or $n_G(t)=O(1)$ (bounded growth) are the only two possible cases \cite{CrespiACM2008}.
For unbounded growth to occur there must exist three paths $\pi_1,\pi_2,\pi_3:[0, T]\to V$ such that:
\begin{enumerate}
	\item All three paths permit the same color sequence: $\ell(\pi_1(i)) = \ell(\pi_2(i)) = \ell(\pi_3(i))$ $\forall i\in[0, T]$
	\item $\pi_1$ is a cycle of period $T$: $\pi_1(0) = \pi_1(T) = a$
	\item $\pi_2$ is a cycle of period $T$ which is distinct from $\pi_1$: $\pi_2(0) = \pi_2(T) = b\neq a$
	\item $\pi_3$ is a path of length $T+1$ which starts at the beginning of $\pi_1$ and ends at the beginning of $\pi_2$: $\pi_3(0)=a$, $\pi_3(T)=b$
\end{enumerate}

This structure is shown for an abstract model in Figure~\ref{fig:graphBranchBridge}.
\begin{figure}
	\centering
	\includegraphics{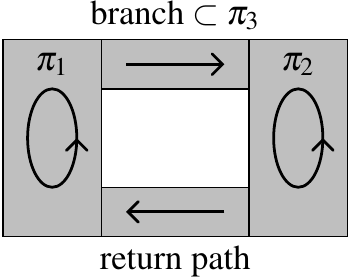}
%	\begin{tikzpicture}[>={angle 90[scale=1]}]
%		\draw[fill=lightgray] (0,0) rectangle ++(1,2);
%		\draw[thick, decoration={markings, mark=at position 0.0 with {\arrow[thick]{>}}}, postaction={decorate}] (0.5,1) ellipse [x radius = 0.25cm, y radius = 0.5cm];
%		\node at (0.5,1.75) {$\pi_1$};
%  
%		\draw[fill=lightgray] (1,2) rectangle ++(1.5,-0.5);
%		\draw[thick, ->] (1.25,1.75) -- (2.25,1.75);
%		\node at (1.75,2.25) {$\text{branch}\subset \pi_3$};
%  
%		\draw[fill=lightgray] (2.5,0) rectangle ++(1,2);
%		\draw[thick, decoration={markings, mark=at position 0.0 with {\arrow[thick]{>}}}, postaction={decorate}] (3.0,1) ellipse [x radius = 0.25cm, y radius = 0.5cm];
%		\node at (3.0,1.75) {$\pi_2$};
%  
%		\draw[fill=lightgray] (1,0) rectangle ++(1.5,0.5);
%		\draw[thick, ->] (2.25,0.25) -- (1.25,0.25);
%		\node at (1.75,-0.25) {return path};
%	\end{tikzpicture}
	\caption{Components of a strongly connected model with paths $\pi_1$, $\pi_2$, and $\pi_3$. The ``branch'' is the portion of $\pi_3$ which is part of neither $\pi_1$ or $\pi_2$. The ``return path'' must exist for the model to be strongly connected. This construct is relevant to Theorem \ref{thm:bounded}.}
	\label{fig:graphBranchBridge}
\end{figure}
The ``branch'' contains whichever nodes and edges in the path $\pi_3$ from $\pi_1$ to $\pi_2$ are not contained in the cycles $\pi_1$ and $\pi_2$.

Suppose now that the model is \emph{strongly connected}, meaning that every node is reachable from every other node.
For this condition to hold, there must be at least one ``return path'' to get from $\pi_2$ back to $\pi_1$. Using this structure, we can prove the following theorem.

\begin{theorem}
\label{thm:bounded}
In a strongly-connected trackable model $G$, $n_G(t)=O(1)$.
\end{theorem}

\begin{proof}
Suppose a strongly-connected trackable model $G$ has $n_G(t)=O(t^k)$ with $k >0$, and hence has the three paths $\pi_1$, $\pi_2$, $\pi_3$ described above.
In such a model, it is always possible to construct a pair of intersecting cycles $\pi_4$, $\pi_5$ having the same color sequence according to the following recipe:
\begin{enumerate}
	\item $\pi_4$ and $\pi_5$ start at a node in $\pi_2$ which connects to a return path.
	\item $\pi_4$ and $\pi_5$ both traverse the return path and begin to follow $\pi_1$.
	\item $\pi_4$ and $\pi_5$ traverse $\pi_1$ until the branch to $\pi_3$ is reached.
	\item $\pi_4$ stays on $\pi_1$ while $\pi_5$ follows $\pi_3$ then begins to follow $\pi_2$.
	\item $\pi_4$ reaches the branch a second time and follows $\pi_3$ while $\pi_5$ continues to follow $\pi_2$.
	\item $\pi_4$ and $\pi_5$ follow $\pi_2$ until they reach the starting node in $\pi_2$, thus completing the cycles.
\end{enumerate}
Note that it does not matter how long the return path is, if there are multiple return paths, or at what point the return path connects between $\pi_1$ and $\pi_2$, because $\pi_4$ and $\pi_5$ traverse a given return path together.
Also, because $\pi_3$ has the same length as $\pi_1$ and $\pi_2$, $\pi_4$ and $\pi_5$ will reconnect at the same node in $\pi_2$.
Because a trackable model cannot have intersecting cycles which permit the same color sequence, this contradicts our assumptions.
Therefore, a strongly-connected trackable model cannot have $\pi_1$, $\pi_2$, and $\pi_3$ and hence must have $n_G(t)=O(1)$.
\end{proof}

To illustrate this proof, consider the model shown in Figure~\ref{fig:branchExample}.
\begin{figure}
	\centering
	\includegraphics{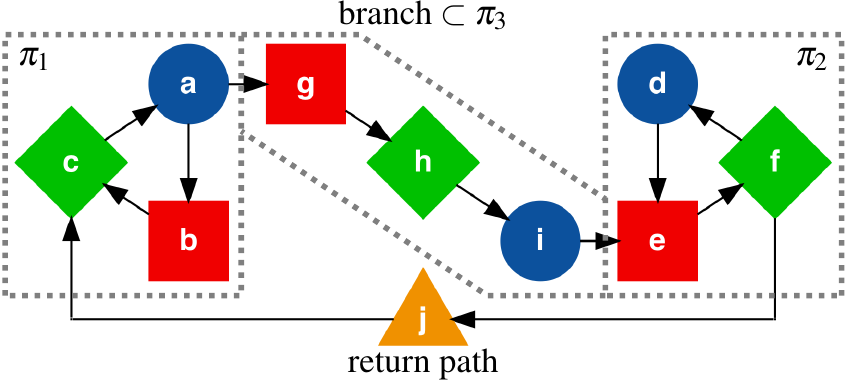}
%	\begin{tikzpicture}
%		\node[inner sep=0pt] at (0,0) {\includegraphics[scale=0.15]{branchExample}};
%		\draw[ultra thick, dotted, color=C7] (-4.25,-0.9) rectangle (-1.85,1.75);
%		\node at (-3.95,1.5) {$\pi_1$};
%		\draw[ultra thick, dotted, color=C7] (4.25,-0.9) rectangle (1.85,1.75);
%		\node at (3.95,1.5) {$\pi_2$};
%		\draw[ultra thick, dotted, color=C7] (-1.85,1.75) -- (-0.65,1.75) -- (1.85,0.08) (1.85,-0.9) -- (0.65,-0.9) -- (-1.85,0.77);
%		\node at (0,1.95) {$\text{branch}\subset\pi_3$};
%		\node at (0,-1.6) {return path};
%	\end{tikzpicture}
	\caption{Strongly-connected model with $\pi_1$, $\pi_2$, and $\pi_3$. To make the example more general, the branch and return path were constructed to consist of more than a single edge.}
	\label{fig:branchExample}
\end{figure}
This model has $\pi_1$, $\pi_2$ and $\pi_3$ as listed in Table~\ref{tab:polyCycList}.
\begin{table}
	\centering
	\caption{Cycles $\pi_1$, $\pi_2$ and path $\pi_3$ associated with Figure~\ref{fig:branchExample}.}
	\begin{tabular}{rCCCCCCC}
		\toprule
		$t$ & 0 & 1 & 2 & 3 & 4 & 5 & 6\\
		Color &\bluecircle & \redsquare & \greendiamond & \bluecircle & \redsquare & \greendiamond & \bluecircle\\
		\midrule
		$\pi_1$ & a & b & c & a & b & c & a\\
		$\pi_2$ & d & e & f & d & e & f & d\\
		$\pi_3$ & a & g & h & i & e & f & d\\
		\bottomrule
	\end{tabular}
	\label{tab:polyCycList}
\end{table}
Following the recipe from the proof, we can construct the pair of cycles given in Table~\ref{tab:intCycles}.
\begin{table*}
	\centering
	\caption{Cycles $\pi_4$ and $\pi_5$ associated with Figure~\ref{fig:branchExample}.}
	\begin{tabular}{rCCCCCCCCCCCC}
		\toprule
		$t$ & 0 & 1 & 2 & 3 & 4 & 5 & 6 & 7 & 8 & 9 & 10 & 11\\
		Color & \greendiamond & \orangetriangle & \greendiamond & \bluecircle & \redsquare & \greendiamond & \bluecircle & \redsquare & \greendiamond & \bluecircle & \redsquare & \greendiamond\\
		\midrule
		$\pi_4$ & f & j & c & a & b & c & a & g & h & i & e & f\\
		$\pi_5$ & f & j & c & a & g & h & i & e & f & d & e & f\\
		\bottomrule
	\end{tabular}
	\label{tab:intCycles}
\end{table*}
The color sequence given in the table permits two paths which begin and end at the same node, which means the size of the hypothesis set will double each time the color sequence is observed -- this exponential growth in the number of hypotheses means that the model is not trackable.

This theorem has good implications for the feasibility of the node sequence reconstruction problem in strongly-connected trackable models.
Specifically, the basic implication of trackability is that there is a possibility for \emph{unbounded} polynomial growth -- a tracking system with finite memory monitoring a trackable model for a long time would require a mechanism for pruning unlikely but permissible hypotheses.
But if a trackable model is also strongly connected we only ever have to keep track of a finite number of hypotheses, regardless of the observation length.

In a trackable model which is not strongly connected, we can discuss the implications of this theorem by borrowing some terminology from the study of Markov chains \cite{Bertsekas2008}.
The set of nodes which can be reached from node $i$ are called the nodes \emph{accessible} from $i$.
A node $i$ is called \emph{recurrent} if, for every node $j$ which is accessible from $i$, $i$ is accessible from $j$.
A node which is not recurrent is called \emph{transient}.
The set of nodes reachable from a recurrent node form a \emph{recurrent class}.
Note that a model could have multiple disjoint recurrent classes.

Using this terminology, we can state the following corollary to Theorem~\ref{thm:bounded}.

\begin{corollary}
	\label{cor:transient}
	In a trackable model $G$ for which $n_G(t)=O(t^k)$ with $k>0$, the cycle $\pi_1$ involves only transient nodes.
\end{corollary}

\begin{proof}
	For $n_G(t)=O(t^k),~k>0$ to hold, the paths $\pi_1$, $\pi_2$, and $\pi_3$ described above must exist.
	But, for the model to be trackable, there can be no return path from the nodes in $\pi_2$ back to the nodes in $\pi_1$, otherwise the recipe in the proof to Theorem~\ref{thm:bounded} would lead to a contradiction.
	Therefore, the nodes in $\pi_1$ cannot be accessible from the nodes in $\pi_2$.
	But, the nodes in $\pi_2$ are accessible from the nodes in $\pi_1$, so the nodes in $\pi_1$ are transient.
\end{proof}

The behavior of a weak model or a Markov chain can be divided into two parts: the burn-in period when it visits the transient nodes, and the steady-state period when it has entered a recurrent class.
Corollary~\ref{cor:transient} implies that unbounded polynomial growth arises only from the transient burn-in period.
It may or may not be possible, however, to tell when the system has transitioned between burn-in and steady-state behavior.
Consider the model in Figure~\ref{sf:ex1}: it is never possible to tell when the system has transitioned from $a$ to $b$, and hence $n_G(t)$ increases linearly in $t$.
\begin{figure}
	\centering
	\subfloat[]{\includegraphics[scale=0.15]{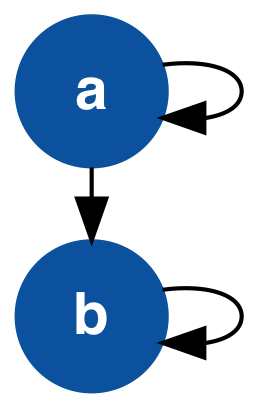}\label{sf:ex1}}
	\quad
	\subfloat[]{\includegraphics[scale=0.15]{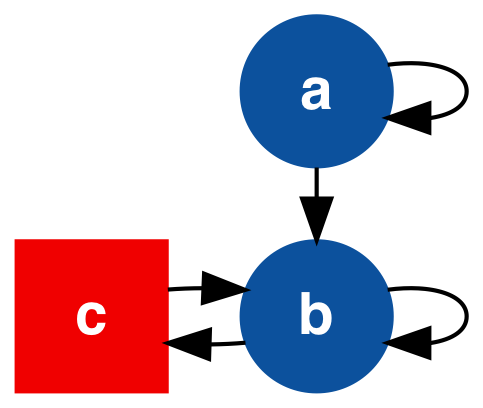}\label{sf:ex2}}
	\caption{Model for which $n_G(t)=O(t^k)$, $k>0$ and it \protect\subref{sf:ex1} is not and \protect\subref{sf:ex2} is eventually possible to tell that the system has entered a recurrent class.}
	\label{fig:burnExample}
\end{figure}
In the model in Figure~\ref{sf:ex2}, however, once \redsquare is observed we know the system is at $c$ and hence has entered the recurrent class $\{b,c\}$ and will stay there -- for this example, $n_G(t)$ remains fixed at whatever size it was before \redsquare was observed.

In cases where it is not possible to know when the system has entered a recurrent class, it is necessary to characterize the typical length of the burn-in period so that unlikely trajectories which spend excessive time in transient nodes may be removed from the hypothesis set.
In a Markov chain the expected number of time steps before a recurrent class is entered when starting from state $i$ is called the mean absorption time $\mu_i$, and is found by solving the following system of equations \cite{Bertsekas2008}:
\begin{gather}
	\mu_i = \begin{cases}
		0, &\text{$i$ recurrent}\\
		1 + \sum_{j\in V} P_{ij}\mu_j, &\text{$i$ transient}
	\end{cases},
\end{gather}
where $P_{ij}=P(X_{t+1}=j|X_t=i)$ is the probability of transitioning from node $i$ to node $j$.
For the model in Figure~\ref{sf:ex1}, we have
\begin{gather}
	\mu_a = \frac{1}{1-P_{aa}}.
\end{gather}
For example, if $P_{aa}=0.9$ then $\mu_a=10$: we expect to have reached $b$ after 10 steps or so, after which we can remove the highly unlikely trajectories of the form $(a,a,\dots)$ from the hypothesis set.

\section{Reconstructability}
\label{sec:recon}
The hypothesis set grows whenever a node which has \emph{same-colored out-neighbors} is encountered.
Formally, given $(v,v_1),(v,v_2)\in E$, we say that $v_1$ and $v_2$ are same-colored out-neighbors of $v$ if $\ell(v_1)=\ell(v_2)$.\footnote{The appropriate generalization for multi-colored graphs is that same-colored out-neighbors are pairs of out-neighbors which have the \emph{potential} to grow the hypothesis set because they share at least one color: $L(v_1)\cap L(v_2)\neq\emptyset$.}
If there are $m$ out-neighbors $v_i$ of $v$ such that $\ell(v_i)$ is the same for all $i$, we refer to $m$ as the \emph{multiplicity} of the same-colored out-neighbors.
If a system starts out at a known node, the nodes visited can be reconstructed unambiguously until a same-colored out-neighbor is encountered.
Once a same-colored out-neighbor is encountered the hypothesis set branches, and a key question is whether or not we can eventually determine which branch was taken.
The following theorem shows that this reconstruction is always possible.

\begin{theorem}
	\label{thm:recon}
	In a strongly-connected trackable model, it will eventually be possible to identify which branch was taken at a node that has same-colored out-neighbors.
\end{theorem}

\begin{proof}
	Suppose the theorem does not hold. Because the model is assumed to be strongly connected and trackable, a node $v$ with same-colored out-neighbors of multiplicity $m$ can be revisited an infinite number of times during an observation of infinite length. Every time the same-colored out-neighbors of $v$ are encountered, the number of hypotheses which go through $v$ at that time is multiplied by $m$. If it were not possible to identify which branch was taken before the system encounters the same-colored out-neighbors again, the number of such hypotheses would not have shrunk back down before this multiplication occurs. This would lead to exponential growth in the size of the hypothesis set, contradicting the assumption that the model is trackable.
\end{proof}

This has good implications for reconstruction of the sequence of nodes visited in a strongly-connected trackable model.
Specifically, if it is possible to uniquely identify the node the system visits at some time (such as when the system starts at a known node), then it will eventually be possible to uniquely identify all subsequent nodes visited by the system.

To quantify what ``eventually'' means in a Markov chain, consider a node $v$ which has same-colored out-neighbors.
We know that the question of which branch was taken must be resolved by the time the system transitions out of $v$ a second time.
(It may of course be much quicker, depending on the specific structure of the model -- see below for further discussion of this point.)
The expected time to travel from a node $v$ back to itself is called the mean recurrence time $t_v^*$.
This can be computed by first computing the mean first passage times $t_i$ (i.e., the expected time to first visit $v$ when starting from $i$) by solving the following system of equations \cite{Bertsekas2008}:
\begin{gather}
	t_i = \begin{cases}
		0, & i=v\\
		1 + \sum_{j\in V} P_{ij}t_j, & i\neq v
	\end{cases},
\end{gather}
then substituting the result into
\begin{gather}
	t_v^* = 1 + \sum_{j\in V}P_{vj}t_j.
\end{gather}
(Note that the calculation of mean first passage times is very similar to the calculation of mean absorption times in the previous section. Mean absorption times are the expected time until the first passage of \emph{any} recurrent node, whereas mean first passage times are the expected time before the first passage of the \emph{specific} node $v$.)

As an example, consider the model shown in Figure~\ref{sf:reconstructEx}.
\begin{figure}
	\centering
	\subfloat[Model]{\includegraphics[scale=0.15]{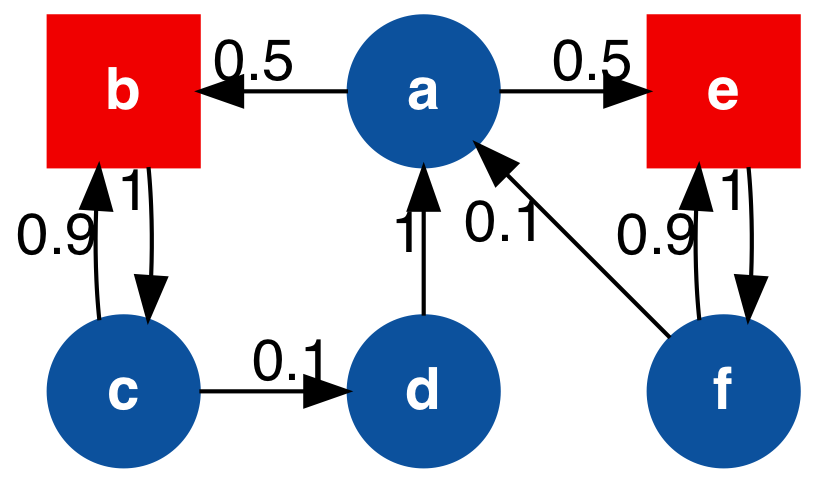}\label{sf:reconstructEx}}
	\quad
	\subfloat[Simulation results]{\includegraphics{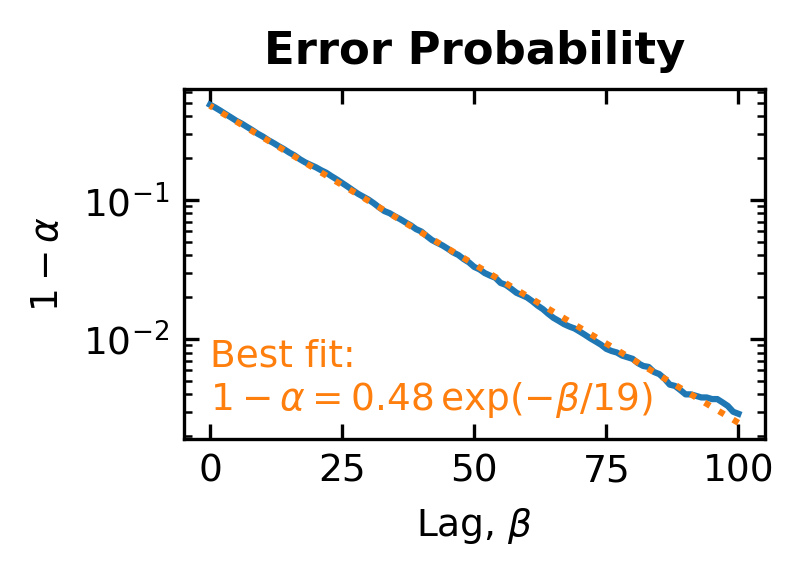}\label{sf:reconstructExRes}}
	\caption{\protect\subref{sf:reconstructEx} Model to illustrate the mean time to reconstruct which branch out of $a$ is taken. The transition probabilities used for the example are indicated on each edge. The nodes are single-colored, so the emission matrix needed to express this as a discrete-symbol hidden Markov model only has one non-zero element in each row. \protect\subref{sf:reconstructExRes} Error probability $1-\alpha$ as a function of lag $\beta$. The results (blue solid) are shown on a log-linear scale together with the best fit line (orange dotted). The behavior is well-described by an exponential with a time constant of 19.}
\end{figure}
Once the system leaves $a$, it could spend an arbitrary amount of time in the cycles of the form (\redsquare, \bluecircle, \redsquare).
But, when the system returns to $a$, either the sequence (\redsquare, \bluecircle, \bluecircle, \bluecircle, \redsquare) or the sequence (\redsquare, \bluecircle, \bluecircle, \redsquare) is seen, which indicates whether the system took the left- or right-hand branch, respectively.
For the model in Figure~\ref{sf:reconstructEx},
\begin{gather}
	t_a^* = 1 + P_{ab}\cdot\left(\frac{2}{P_{cd}} + 1\right) + P_{ae}\cdot\frac{2}{P_{fa}} = 21.5.
\end{gather}

In order to illustrate this, we simulated \num{10000} traversals, each consisting of 200 time steps.
(We used many traversals in order to obtain a good estimate the error rate as a function of lag.)
In order to avoid issues with reconstruction of the first couple nodes, we started each traversal at $a$.
We reconstructed the most likely node sequence using the Viterbi algorithm and computed the accuracy $\alpha=P(\hat{X}_{t-\beta}=X_{t-\beta})$ as a function of the lag $\beta$, where $\hat{X}_{t}$ and $X_t$ are the estimated and true nodes at time $t$, respectively.
The results are shown in Figure~\ref{sf:reconstructExRes}.
The behavior of the error probability $1-\alpha$ is linear when shown on a log-linear scale, indicating that it is described by a function of the form
\begin{gather}
	1 - \alpha = A\exp(-\beta / \tau),
\end{gather}
where $\tau$ is the time constant and $A$ is the error probability at $\beta=0$.
The best fit values are $A=0.48$, $\tau=19\approx t_a^*$: the time constant of the recovery in accuracy following a same-colored out-neighbor is roughly equal to the mean recurrence time, as expected.
Furthermore, the accuracy for $\beta=0$ is $1-A= 52\%$, consistent with the 50/50 chance of being in the left-hand or right-hand branch.

Note, however, that the mean recurrence time can be a very pessimistic estimate of the time needed to determine which branch was taken for two reasons.
First, the structure of the model may permit unique reconstruction of which branch was taken well before the same-colored out-neighbors are revisited.
If, for example, node $c$ in the present example were changed to emit \greendiamond, the ambiguity caused by the color sequence (\bluecircle, \redsquare) would be resolved on the very next step when either \greendiamond or \bluecircle is observed.
Even if the structure of the model does not permit early reconstruction in this manner, however, the transition dynamics themselves may facilitate better performance than would be expected from the mean recurrence time alone.
For example, suppose that $P_{fa}$ is reduced to 0.01 in the present example.
This yields $t_a^* = 111.5$, but repeating the above analysis gives $\tau=21$, $A=0.08$.
This is because the average time spent in the two branches is now dramatically different, so sufficiently long sequences of the form $(\redsquare,\bluecircle,\redsquare,\bluecircle,\dots)$ are correctly inferred to have come from the right-hand branch even before node $a$ is revisited.

\section{Maximum Size of the Hypothesis Set}
\label{sec:maxSize}
Having established that strongly-connected trackable models have a bounded number of hypotheses, it is of interest to know what that bound is.
The main theorem of this section stems from the observation that the size of the hypothesis set only increases when:
\begin{enumerate}
	\item the first color observed can be emitted by more than one node, or
	\item a same-colored out-neighbor is encountered (as discussed in the previous section).
\end{enumerate}
Before stating and proving the theorem, we need the following corollary to Theorem~\ref{thm:recon}.

\begin{corollary}
	In a strongly-connected trackable model for which observations start from a known node, there can be at most one hypothesis which visits a given node at a given time.
	\label{cor:singlePath}
\end{corollary}
\begin{proof}
	Suppose the corollary does not hold. Consider a color sequence such that, at some point after the known starting point, multiple hypotheses are created because of same-colored out-neighbors. Furthermore, suppose that the color sequence permits a pair of these hypotheses to visit the same node at the same time. Because the subsequent transitions do not depend on any of the previous transitions (Markov property), this means that which branch was taken at the node with same-colored out-neighbors can never be reconstructed, contradicting Theorem~\ref{thm:recon}.
\end{proof}

As an alternate proof of the corollary, note that color sequence described above would permit the construction of two paths from the (known) starting node back to itself.
Therefore, a model which permits such a color sequence lacks the unique path property, violating the assumption that the model is trackable.

Note that a given node may have multiple sets of same-colored out-neighbors.
Let $M_v$ be the maximum multiplicity of the sets of same-colored out-neighbors at node $v$, where $M_v=1$ for a node which lacks same-colored out-neighbors.
As an example, for the model in Figure~\ref{sf:reconstructEx}, $M_a=2$ and $M_c=1$.

In addition, let $K$ to be the largest number of same colored nodes in the model.
For example, for the model in Figure~\ref{sf:reconstructEx}, $K=4$.

\begin{figure}
	\centering
	\includegraphics[width=\columnwidth]{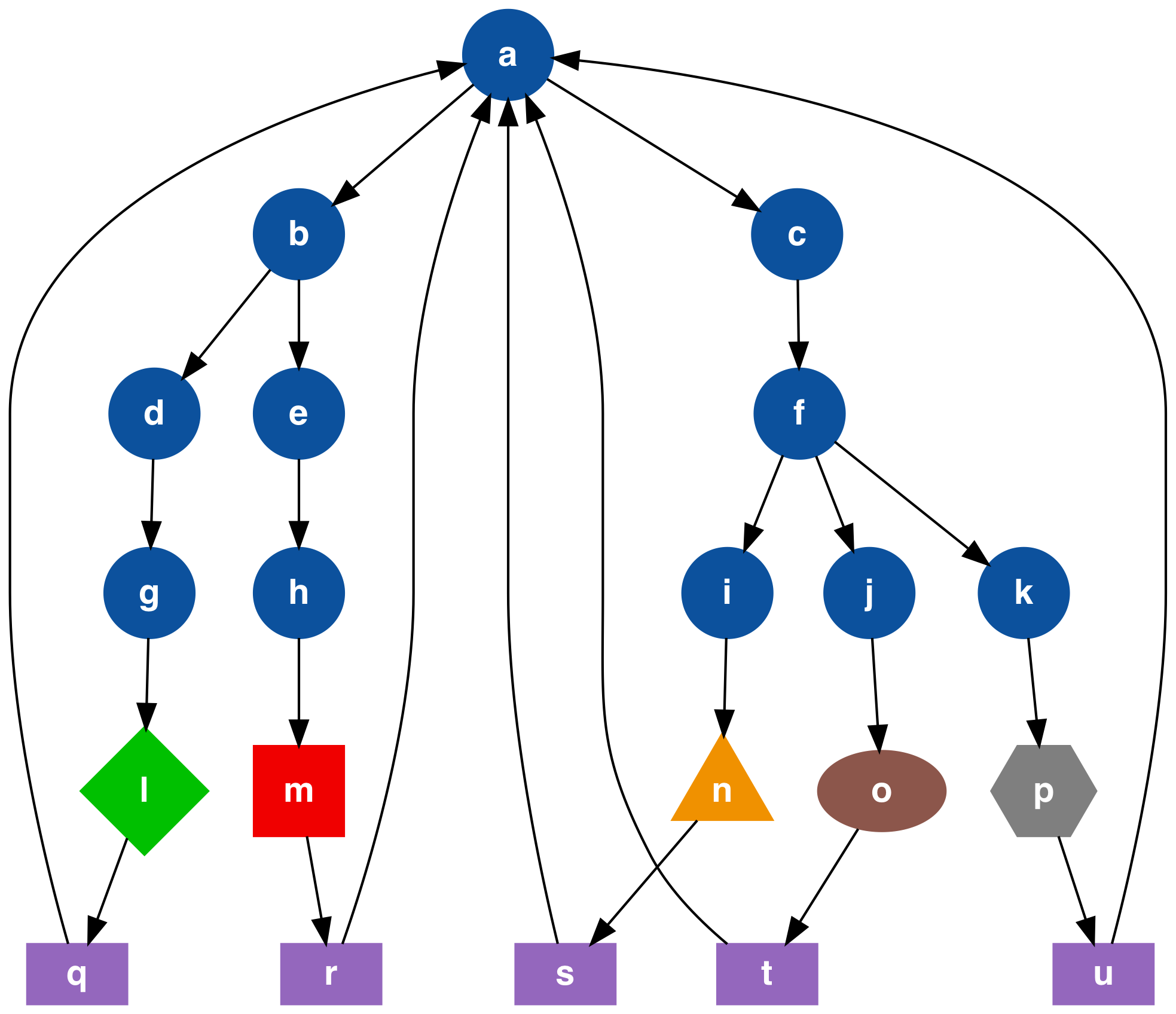}
	\caption{This model illustrates the bounds developed in Theorem~\ref{thm:bound}.}
	\label{fig:boundDemo}
\end{figure}

\begin{theorem}
	In a strongly-connected trackable model, $G$,
	\begin{gather}
		n_G(t)\leq K \left( 1 + \sum_{v\in V} (M_v - 1)\right).
	\end{gather}
	When the observations start from a known node,
	\begin{gather}
		n_G(t)\leq 1 + \sum_{v\in V} (M_v - 1).
	\end{gather}
	\label{thm:bound}
\end{theorem}
\begin{proof}
	We start with the case in which the observations begin at a known node.
	Consider a color sequence with the following properties:
	\begin{enumerate}
		\item Every node which has same-colored out-neighbors is visited by a path consistent with the color sequence.
		\item The paths consistent with the color sequence include transitions into the highest-multiplicity set of same-colored out-neighbors at each node.
		\item The paths consistent with the color sequence visit every node with same-colored out-neighbors on the path before the reconstruction promised by Theorem~\ref{thm:recon} occurs for any of the same-colored out-neighbors encountered.
	\end{enumerate}
	Clearly this color sequence will generate the largest hypothesis set possible.
	If such a color sequence does not exist, then the maximum $n_G(t)$ will be smaller than the bound obtained by assuming such a color sequence exists, thus accounting for the less-than part of the theorem.
	
	To obtain the equal-to part, assume that such a color sequence exists and keep a running tally of $n_G(t)$ as it is processed.
	When starting from a known node, $n_G(1)=1$.
	At each same-colored out-neighbor, the number of hypotheses \emph{through the node with same-colored out-neighbors} is multiplied by $M_v$.
	Corollary~\ref{cor:singlePath} means that there is only ever a single hypothesis going into the node with same-colored out-neighbors at a given time, so we add $M_v - 1$ to our running tally to replace the single hypothesis with the $M_v$ new hypotheses.
	Because $M_v=1$ for nodes with no same-colored out-neighbors, we can simply sum over all nodes to obtain the expression given.
	
	Returning to the case in which the observations do not begin at a known node, note that there are at most $K$ nodes with the color first observed and, for each of those possible starting nodes, there are at most $1 + \sum_{v\in V} (M_v - 1)$ hypotheses as we have just established.
	As a result, there are at most $ K \left( 1 + \sum_{v\in V} (M_v - 1)\right)$ hypotheses even with the ambiguity introduced by not knowing the node when the tracking began.
\end{proof}

The difference between knowing the starting node and not knowing the starting node
is illustrated concretely using the model in Figure~\ref{fig:boundDemo}.
Note that nodes can have multiple same-colored \emph{in}-neighbors.
If observations start when the system is at an unknown same-colored in-neighbor, it will never be possible to identify the first node that was visited.
This would cause $n_G(t)$ to be multiplied by the number of hypotheses generated by this initial ambiguity.
For the model in Figure~\ref{fig:boundDemo}, Theorem~\ref{thm:bound} indicates that $n_G(t) \leq 5$ as long as observations start at a known node.
This bound is realized by the color sequence $(\bluecircle, \bluecircle, \bluecircle, \bluecircle)$.
But, suppose instead that observations start at an unknown one of the same-colored in-neighbors of $a$ (i.e., one of $\{q, r, s, t, u\}$).
This results in the color sequence $(\purplerect, \bluecircle, \bluecircle, \bluecircle, \bluecircle)$, which in turn yields $n_G(4)=25$ because it is not possible to determine which of the \purplerect nodes was visited initially.

Note that the unknown starting point bound is only tight if the color sequence used in the proof of Theorem~\ref{thm:bound} exists and the initial ambiguity from starting at one of the $K$ same-colored nodes is never resolved.
For example, the bound of $n_G(t)\leq 5K=55$ is not tight for the model in Figure~\ref{fig:boundDemo} because the elevenfold ambiguity in an initial observation of \bluecircle is resolved by the time any of the non-\bluecircle nodes are visited.
In general, there are several circumstances under which the initial ambiguity may not be resolved and the bound can be tight:
\begin{enumerate}
	\item The $K$ same-colored nodes could be same-colored in-neighbors of some node.
	\item There could be a color sequence which permits paths from each of the $K$ same-colored nodes to some node. This is a multi-step generalization of the same-colored in-neighbors in the previous case.
	\item The model could be symmetric (i.e., $|\Aut(G)|>1$, where $\Aut(G)$ is the automorphism group of model\footnote{Note that some definitions of automorphisms of graphs ignore the coloring. Here we define an automorphism of a weak model to be a permutation $\sigma$ of the nodes such that $(u,v)\in E \iff (\sigma(u),\sigma(v))\in E$ and $L(u)=L(\sigma(u))$.} $G$): if the $K$ same-colored nodes are not fixed points of the automorphisms then it will not be possible to resolve the initial ambiguity.
\end{enumerate}
As an example of the last case, the model in Figure~\ref{fig:symmetric} has $|\Aut(G)|=K=2$ and no same-colored out-neighbors, so the bound of $n_G(t)\leq 2$ is tight.

\section{Entropy Rates of Weak Models}
\label{sec:entropy}

This section considers the conditional entropy rate of the node sequence given the color sequence.
As entropy rates are determined by the transition and emission probabilities, we consider any given weak model $G=(V, E, L, \Phi)$ in this section to have been converted to a lumped Markov chain through assignment of transition probabilities $P_{ij}$ such that $(i,j)\in E\iff P_{ij}>0$.
We can then speak of the conditional entropy rate of a weak model by establishing inequalities which hold for any assignment of probabilities which satisfies this condition.

The node sequence, $X_{[T]}=(X_1, X_2, \dots, X_T)$, is a stationary, Markovian stochastic process.
The color sequence, $Y_{[T]}=(Y_1, Y_2, \dots, Y_T)$, is also a stationary stochastic process which may or may not be Markovian \cite{GurvitsLAA2005,geiger2014lumpings}.
The entropy rate of stochastic process $X_t$ conditional on stochastic process $Y_t$ is
\begin{gather}
	H(X|Y)= \lim_{T \rightarrow \infty} \frac{1}{T} H(X_{[T]}|Y_{[T]}),\label{eq:condEntRate}
\end{gather}
where $H(X_{[T]}|Y_{[T]})$ is the classical Shannon conditional entropy for possible realizations of $X_{[T]}$ and $Y_{[T]}$ \cite{cover2012elements}.
Specifically,
\begin{align}
	H(X_{[T]}|Y_{[T]}) &= - \sum_{y \in \cal{Y}_{[T]}} p(y) \sum_{x \in \cal{X}_{[T]}}p(x|y)\log_2(p(x|y))\label{eq:condEnt}\\
	&= \sum_{y \in \cal{Y}_{[T]}} p(y) H(X_{[T]}|Y_{[T]}=y),
\end{align}
where $\cal{X}_{[T]}$ and $\cal{Y}_{[T]}$ are the sample spaces of the
random variables $X_{[T]}$ and $Y_{[T]}$, respectively.
The conditional entropy rate corresponds to the average amount of information about $X_t$ which is lost per time step when only $Y_t$ is observed.

Note that trackability is equivalent to having an infinite ``split-merge index'' as defined in the recent work on lumped Markov chains by Geiger and Temmel \cite{geiger2014lumpings}: the non-existence of differing realizable trajectories with the same start and end points is, for the irreducible case considered by Geiger and Temmel, equivalent to the unique path property which characterizes trackability.
Theorem 1 of \cite{geiger2014lumpings} shows that the conditional entropy rate $H(X|Y)$
of the state sequence process conditioned on the observation sequence is
strictly positive for a lumped Markov chain if and only if the corresponding
weak model is untrackable. In other words, for an untrackable model
there is always positive uncertainty about the process' state
however long the observation sequence happens to be.

Moreover, Geiger and Temmel also show that if there are only a bounded number of
hypotheses per observation sequence, then the conditional entropy rate is $H(X|Y)=0$.
Because their work assumes that the underlying Markov chain is irreducible (equivalent to our constraint of strong connectivity),
they are silent about the case when the hypothesis growth is unbounded
polynomial.  As shown above, a weak model that is trackable but has
unbounded worst case hypothesis growth cannot be strongly connected/irreducible.
The following theorem addresses this case.

\begin{theorem}
	A model which has (potentially unbounded) polynomial growth in the hypothesis set, $n_G(T)=O(T^k)$ with $k\geq0$, has $H(X|Y)=0$.
	\label{thm:entrate}
\end{theorem}
\begin{proof}
Recall Equation~(\ref{eq:condEnt}).
If the cardinality of $\cal{X}_{[T]}$ is bounded by a polynomial $cT^k$
for every observation sequence, then the largest entropy
\begin{gather}
\sum_{x \in \cal{X}_{[T]}}
p(x|y)\log_2(p(x|y)) = H(X_{[T]}|Y_{[T]}=y)
\end{gather}
occurs for the uniform distribution over $cT^k$ elements so that
\begin{align}
H(X_{[T]}|Y_{[T]}=y) &=  - \sum_{x \in \cal{X}_{[T]}} p(x|y) \log_2(p(x|y)) \\
&\leq cT^k \cdot \frac{1}{cT^k} \log_2 (cT^k)\\
&=  \log_2 (c) + k \log_2(T).
\end{align}

Consequently, 
\begin{align}
H(X|Y) &= \lim_{T \rightarrow \infty} \frac{1}{T} H(X_{[T]}|Y_{[T]})\\
&=  \lim_{T \rightarrow \infty} \frac{1}{T} \sum_{y \in \cal{Y}_{[T]}} p(y) H(X_{[T]}|Y_{[T]}=y)\\
&\leq \lim_{T \rightarrow \infty} \frac{1}{T} \sum_{y \in \cal{Y}_{[T]}}  p(y) \left( \log_2 (c) + k \log_2(T) \right)\\
&= \lim_{T \rightarrow \infty} \frac{ \log_2 (c) + k \log_2(T)}{T} =0.
\end{align}
Thus, if hypothesis growth is bounded by a polynomial then $H(X|Y)=0$.
\end{proof}

To gain some intuition for the implications of this theorem, note that there are models for which $H(X|Y)=0$ but it is still impossible to uniquely reconstruct $X$ from $Y$.
An example of a strongly-connected model for which this is the case is shown in Figure~\ref{fig:symmetric}.
\begin{figure}
	\centering
	\includegraphics[scale=0.15]{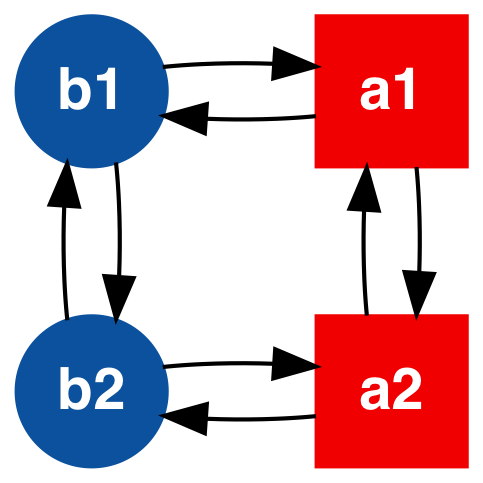}
	\caption{Model from Figure~4 of \cite{geiger2014lumpings}. The model is trackable and hence the conditional entropy rate is $H(X|Y)=0$ but it is not possible to uniquely reconstruct $X$ from $Y$ without additional constraints on the distribution of initial states.}
	\label{fig:symmetric}
\end{figure}
Because it is trackable, this model has $H(X|Y)=0$.
But if observations do not start at a known node the hypothesis set always has size two.
An intuitive way of thinking about this case is that, assuming all nodes are equally likely starting points, which of the two hypotheses is correct represents one bit of uncertainty: $H(X_{[T]}|Y_{[T]})=\SI{1}{bit}$, independent of $T$.
Because this uncertainty does not grow as $T\to\infty$, it is taken to zero by the $1/T$ term inside the limit of Equation~(\ref{eq:condEntRate}).

For a less obvious case, Theorem~\ref{thm:entrate} shows that the model in Figure~\ref{sf:ex1} has $H(X|Y)=0$ despite the fact that there is unbounded growth in the number of hypotheses.
An intuitive explanation of the fact that $H(X|Y)=0$ despite the growth in uncertainty at each step is that the number of hypotheses (and hence the conditional entropy) grows at a rate which is sufficiently slow that the $1/T$ term can still take the conditional entropy rate to zero.
In contrast to this case, Theorem 1 of \cite{geiger2014lumpings} shows that the exponential growth in the number of hypotheses for untrackable models is fast enough to result in a nonzero entropy rate.

Note that for $k=0$ (bounded growth) this theorem is more general than Theorem~1 of \cite{geiger2014lumpings} because it applies even if the model is not strongly connected and/or has periodic nodes.
That being said, the theorem is not an ``if and only if'' statement: $H(X|Y)=0$ does \emph{not} imply that $n_G(T)$ is bounded by a polynomial, which is what the following theorem addresses.
\begin{theorem}
	A model for which all recurrent nodes are aperiodic has $H(X|Y)=0$ if and only if all of the recurrent classes are trackable.
	\label{thm:H0recur}
\end{theorem}
\begin{proof}
	Note that the limit in the definition of the conditional entropy rate means that $H(X|Y)$ depends on the long-term, asymptotic behavior of a model.
	Therefore, the transient behavior of the model is irrelevant in determining the conditional entropy rate: while the transient nodes can generate a color sequence which makes the hypothesis set arbitrarily large, we know that, with probability one, the system will eventually end up in one of the recurrent classes.
	Intuitively, the conditional entropy rate will then be dictated by the entropy rates which are realizable by any given recurrent class.
	If the recurrent classes are all trackable, then they will all have zero conditional entropy rate.
	If any recurrent class is not trackable, then there will be a nonzero conditional entropy rate.
	
	To formalize this intuition, we start with the ``if'' part.
	Let the random variable $C_T$ correspond to which recurrent class $X_T$ is in.
	To ensure this is well-defined for all $T$, we take $C_T$ to also have a special value indicating that $X_T$ is a transient node.
	We can then define the stochastic process $C_{[T]}=(C_1, C_2, \dots, C_T)$.
	Because $C_{[T]}$ is completely determined by $X_{[T]}$, we have
	\begin{gather}
		H(X_{[T]}, C_{[T]}|Y_{[T]}) = H(X_{[T]}|Y_{[T]}).
	\end{gather}
	We can then expand this as
	\begin{gather}
		H(X_{[T]}|Y_{[T]}) = H(C_{[T]} | Y_{[T]}) + H(X_{[T]} | Y_{[T]}, C_{[T]}).
	\end{gather}
	
	Noting that we are interested in the limit as $T\to\infty$, $C_{[T]}$ will have a time $T_R$ where it transitions from transient nodes to one recurrent class, which we denote $C$.
	Therefore, any sufficiently long $C_{[T]}$ can be entirely defined by $T_R$ and $C$.
	We can then rewrite the conditional entropy as
	\begin{gather}
		H(X_{[T]}|Y_{[T]}) = H(T_R, C | Y_{[T]}) + H(X_{[T]}|Y_{[T]}, T_R, C).
	\end{gather}
	This expression has a nice interpretation: the first term corresponds to the uncertainty in which recurrent class is entered and when it is entered, while the second term corresponds to the uncertainty in the nodes given the observed colors and the specific recurrent class and transition time.
	
	Letting $X_{a:b}=(X_a, X_{a + 1}, \dots, X_{b-1}, X_b)$, we can write
	\begin{multline}
		H(X_{[T]}|Y_{[T]}, T_R, C) =\\
		H(X_{1:T_R-1}, X_{T_R:T}|Y_{1:T_R-1}, Y_{T_R:T}, T_R, C).
	\end{multline}
	We then have the following inequalities:
	\begin{multline}
		H(X_{[T]}|Y_{[T]}) \leq H(T_R, C|Y_{[T]})\\
		+ H(X_{1:T_R-1}|Y_{1:T_R-1}, Y_{T_R:T}, T_R, C) \\
		+ H(X_{T_R:T}|Y_{1:T_R-1}, Y_{T_R:T}, T_R, C)
	\end{multline}
	\begin{multline}
		\leq H(T_R, C|Y_{[T]}) + H(X_{1:T_R-1}|Y_{1:T_R-1}, T_R, C) + \\
		H(X_{T_R:T}| Y_{T_R:T}, T_R, C).
	\end{multline}
	Taking the limit to obtain the conditional entropy rate, we obtain
	\begin{multline}
		H(X|Y) \leq \lim_{T\to\infty}\frac{1}{T}\Big( H(T_R, C|Y_{[T]}) \\
		+ H(X_{1:T_R-1}|Y_{1:T_R-1}, T_R, C)\\
		+ H(X_{T_R:T}| Y_{T_R:T}, T_R, C) \Big).
	\end{multline}
	The first term has only a polynomial number of possibilities for $T_R$ and a finite, constant number of possibilities for $C$.
	Therefore, the same reasoning used to prove Theorem~\ref{thm:entrate} shows that this term will go to zero.
	The second term does not grow with $T$, and hence will be taken to zero by the $1/T$ in the limit.
	The third term can be expanded as
	\begin{multline}
		H(X_{T_R:T}| Y_{T_R:T}, T_R, C) = \\
		\sum_{\substack{t_r\in\mathcal{T_R},\\c\in\mathcal{C}}}p(t_r, c)H(X_{T_R:T}| Y_{T_R:T}, T_R=t_r, C=c),
	\end{multline}
	where $\mathcal{T_R}$ and $\mathcal{C}$ are the sample spaces of $T_R$ and $C$, respectively.
	Once the limit is taken, each term in the sum corresponds to the conditional entropy rate produced by a given recurrent class $c$.
	Theorem~1 of \cite{geiger2014lumpings} states that this entropy rate is zero for an irreducible, aperiodic, trackable model.
	Therefore, if all recurrent classes are trackable and all recurrent nodes are aperiodic, all terms in the sum are zero and $H(X|Y)=0$.
	
	To prove the ``only if'' part, start with the inequality
	\begin{multline}
		H(X_{[T]}|Y_{[T]}) \geq H(X_{[T]}|Y_{[T]}, T_R, C)\\
		=\sum_{\substack{t_r\in\mathcal{T_R},\\c\in\mathcal{C}}}p(t_r, c)H(X_{[T]}|Y_{[T]},T_R=t_r, C=c).
	\end{multline}
	For each term in the sum, we have
	\begin{multline}
		H(X_{[T]}|Y_{[T]},T_R=t_r, C=c) \geq \\
		\max\Big( H(X_{1:T_R-1}|Y_{[T]},T_R=t_r, C=c), \\
		H(X_{T_R:T}|Y_{[T]},T_R=t_r, C=c) \Big).
	\end{multline}
	When we take the limit as $T\to\infty$, the first term only involves the uncertainty in the finite set of random variables $X_{1:T_R-1}$, and hence will be taken to zero by the factor of $1/T$.
	This means that the lower bound will be determined by $H(X_{T_R:T}|Y_{[T]},T_R=t_r, C=c)$, which is the entropy of the states in the recurrent part conditioned on \emph{all} of the observations.
	Focusing on this term, we then have
	\begin{multline}
		H(X_{T_R:T}|Y_{[T]},T_R=t_r, C=c) \geq\\
		H(X_{T_R:T}|X_{T_R-1}, Y_{1:T_R-1}, Y_{T_R:T},T_R=t_r, C=c).
	\end{multline}
	Because of the conditional independence properties of hidden Markov models, we can drop the conditioning on $Y_{1:T_R-1}$: all of the information the symbols from the transient phase convey about $X_{T_R:T}$ is subsumed by $X_{T_R-1}$.
	So,
	\begin{multline}
		H(X_{[T]}|Y_{[T]},T_R=t_r, C=c) \geq\\
		H(X_{T_R:T}|X_{T_R-1}, Y_{T_R:T},T_R=t_r, C=c)=\\
		\sum_{x\in\mathcal{X}_{\mathcal{T_R}-1}}p(x|T_R=t_r, C=c)\\
		\cdot H(X_{T_R:T}|X_{T_R-1}=x, Y_{T_R:T},T_R=t_r, C=c).
	\end{multline}
	But, $H(X_{T_R:T}|X_{T_R-1}=x, Y_{T_R:T},T_R=t_r, C=c)$ is simply the entropy produced by transitions within recurrent class $c$ with an initial state distribution dictated by the edges from $x$ into $c$.
	Note that Geiger and Temmel assume the initial state distribution is equal to the stationary distribution, but they comment that their results should hold for any initial state distribution \cite{geiger2014lumpings}.
	Indeed, the proof of Theorem~1 of \cite{geiger2014lumpings} does not appear to depend on any specific choice of initial state distribution.
	Therefore, we can take the limit to obtain the conditional entropy rate and apply the results of Theorem~1 of \cite{geiger2014lumpings} to each term in the sum to conclude that the only way for $H(X|Y)>0$ is for at least one of the recurrent classes to be untrackable.
\end{proof}

This theorem implies that models which are untrackable because of pathologies in the transient nodes can still have $H(X|Y)=0$.
This can be understood in terms of the difference between the worst-case and asymptotic average behavior: trackability considers whether there is \emph{any} color sequence which can generate exponential growth in the size of the hypothesis set, whereas the conditional entropy rate captures the fact that, on asymptotic average, such sequences cannot persist infinitely long if they are only generated by transient nodes.
To illustrate this, consider the model in Figure~\ref{fig:untrackH0}.
\begin{figure}
	\centering
	\includegraphics[scale=0.15]{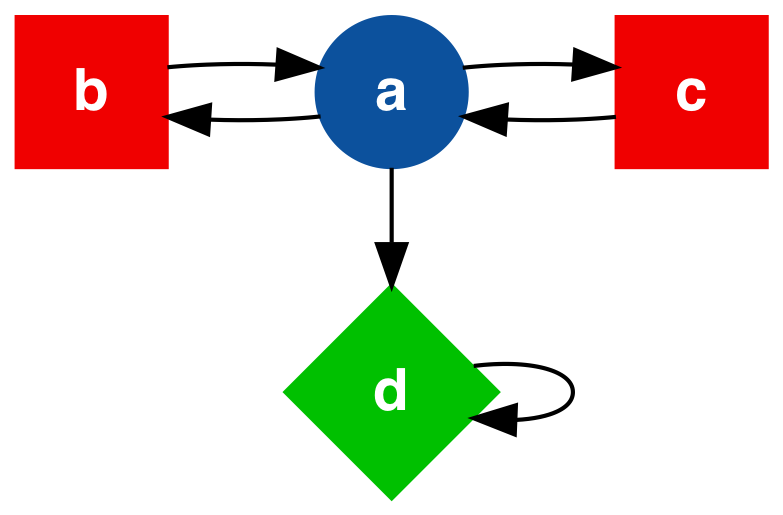}
	\caption{Untrackable model with $H(X|Y)=0$.}
	\label{fig:untrackH0}
\end{figure}
This model has intersecting cycles with the same coloring formed by the transient nodes $\{a, b, c\}$, and therefore experiences exponential worst-case growth of the hypothesis set.
But, the single recurrent node $d$ has a unique color, so the hypothesis set stops growing as soon as \greendiamond is seen.
Given our assumption that the probabilities assigned to the edges are strictly positive, this is guaranteed to happen eventually.
At this point, there are no additional bits of uncertainty added at each time step.

As a less obvious example, suppose the model in Figure~\ref{fig:untrackH0} is modified such that all nodes emit the same color.
Then it is never possible to identify when the recurrent node $d$ has been entered, and yet Theorem~\ref{thm:H0recur} implies that $H(X|Y)=0$.
In this case, the hypothesis set continues to grow exponentially, but the hypotheses which visit the transient nodes many times are increasingly less likely.
Theorem~\ref{thm:H0recur} then indicates that the rate at which such hypotheses become unlikely is sufficient for the conditional entropy rate to be zero.

\section{Summary}
\label{sec:conc}

This paper has presented several new analytic results regarding the ability to reconstruct the sequence of nodes visited by an agent traversing a trackable weak model.
It was found that the node reconstruction problem is substantially more tractable in models that are strongly connected compared to models that have transient nodes.
This is a promising result because trackability imposes fewer constraints on the model's coloring than the other observability classes which offer stronger performance guarantees \cite{ChilenskiTaxon2018}, and hence is more likely to be attainable in practice.
For strongly-connected trackable models, the worst-case size of the hypothesis set is bounded by an easily-computed constant.
Furthermore, when unambiguous real-time tracking is lost because of a same-colored out-neighbor, it will always eventually be possible to determine which branch was taken.
Finally it was shown that the conditional entropy rate, $H(X|Y)$, vanishes if and only if all recurrent classes are trackable.
The implications of these theorems were elaborated in terms of well-known mean absorption and recurrence times of Markov chains that can
be used to estimate how long certain events such as unique node identification can take on average.

These new results identify deep relationships between structural properties of weak models and lumped Markov chains, and their trackability and entropy rates.
Consequently, the results of this paper can be used to concretely inform the design of trackable systems as well as the tracking algorithms used in
applications such as computer program execution monitoring with out-of-band electromagnetic emissions measurements \cite{ChilenskiSPIE2018}.

\section*{Acknowledgements}
Distribution Statement ``A'' (Approved for Public Release, Distribution Unlimited)

The research effort depicted was sponsored by the Air Force Research Laboratory (AFRL) and the Defense Advanced Research Projects Agency (DARPA) under the Leveraging the Analog Domain for Security (LADS) program under contract number FA8650-16-C-7622. In particular, we thank Dr.\ Angelos Keromytis and Mr.\ Ian Crone, the past and present DARPA program managers of LADS, for their encouragement and support throughout the program. 

This research was developed with funding from the Defense Advanced Research Projects Agency (DARPA).

The views, opinions and/or findings expressed are those of the author and should not be interpreted as representing the official views or policies of the Department of Defense or the U.S.\ Government.

\bibliographystyle{IEEEtran}
\bibliography{trackable}

\begin{thebibliography}{10}
\providecommand{\url}[1]{#1}
\csname url@samestyle\endcsname
\providecommand{\newblock}{\relax}
\providecommand{\bibinfo}[2]{#2}
\providecommand{\BIBentrySTDinterwordspacing}{\spaceskip=0pt\relax}
\providecommand{\BIBentryALTinterwordstretchfactor}{4}
\providecommand{\BIBentryALTinterwordspacing}{\spaceskip=\fontdimen2\font plus
\BIBentryALTinterwordstretchfactor\fontdimen3\font minus
  \fontdimen4\font\relax}
\providecommand{\BIBforeignlanguage}[2]{{%
\expandafter\ifx\csname l@#1\endcsname\relax
\typeout{** WARNING: IEEEtran.bst: No hyphenation pattern has been}%
\typeout{** loaded for the language `#1'. Using the pattern for}%
\typeout{** the default language instead.}%
\else
\language=\csname l@#1\endcsname
\fi
#2}}
\providecommand{\BIBdecl}{\relax}
\BIBdecl

\bibitem{RabinerIEEE1989}
L.~R. Rabiner, ``A tutorial on hidden {M}arkov models and selected applications
  in speech recognition,'' \emph{Proceedings of the IEEE}, vol.~77, no.~2, pp.
  257--286, 1989.

\bibitem{CrespiACM2008}
V.~Crespi, G.~Cybenko, and G.~Jiang, ``The theory of trackability with
  applications to sensor networks,'' \emph{ACM Transactions on Sensor Networks
  (TOSN)}, vol.~4, no.~3, p.~16, 2008.

\bibitem{ChilenskiSPIE2018}
M.~Chilenski, G.~Cybenko, I.~Dekine, P.~Kumar, and G.~Raz, ``Control flow graph
  modifications for improved {RF}-based processor tracking performance,'' in
  \emph{Cyber Sensing 2018}, ser. Proc.\ SPIE, vol. 10630, 2018, p. 106300I.

\bibitem{kearns1994cryptographic}
M.~Kearns and L.~Valiant, ``Cryptographic limitations on learning {B}oolean
  formulae and finite automata,'' \emph{Journal of the ACM (JACM)}, vol.~41,
  no.~1, pp. 67--95, 1994.

\bibitem{ChilenskiTaxon2018}
M.~Chilenski, G.~Cybenko, I.~Dekine, P.~Kumar, and G.~Raz, ``Observability
  properties of colored graphs,'' \emph{IEEE Transactions on Network Science
  and Engineering} (Early Access), 2019.

\bibitem{ShengIEEE2005}
Y.~Sheng and G.~V. Cybenko, ``Distance measures for nonparametric weak process
  models,'' in \emph{Systems, Man and Cybernetics, 2005 {IEEE} International
  Conference on}.\hskip 1em plus 0.5em minus 0.4em\relax IEEE, 2005.

\bibitem{JungersDAM2011}
R.~M. Jungers and V.~D. Blondel, ``Observable graphs,'' \emph{Discrete Applied
  Mathematics}, vol. 159, no.~10, pp. 981--989, 2011.

\bibitem{GurvitsLAA2005}
L.~Gurvits and J.~Ledoux, ``{M}arkov property for a function of a {M}arkov
  chain: A linear algebra approach,'' \emph{Linear Algebra and Its
  Applications}, vol. 404, pp. 85--117, 2005.

\bibitem{geiger2014lumpings}
B.~C. Geiger and C.~Temmel, ``Lumpings of {M}arkov chains, entropy rate
  preservation, and higher-order lumpability,'' \emph{Journal of Applied
  Probability}, vol.~51, no.~4, pp. 1114--1132, 2014.

\bibitem{Bertsekas2008}
D.~P. Bertsekas and J.~N. Tsitsiklis, \emph{Introduction to Probability},
  2nd~ed.\hskip 1em plus 0.5em minus 0.4em\relax Athena Scientific, 2008,
  ch.~7.

\bibitem{cover2012elements}
T.~M. Cover and J.~A. Thomas, \emph{Elements of {I}nformation {T}heory}.\hskip
  1em plus 0.5em minus 0.4em\relax John Wiley \& Sons, 2012.

\end{thebibliography}

\begin{IEEEbiographynophoto}{Mark Chilenski}
received the BS degree in aeronautical and astronautical engineering from the University of Washington in 2010 and the PhD degree in nuclear science and engineering from the Massachusetts Institute of Technology in 2016. He is a senior scientist at Systems \& Technology Research LLC. His research interests include machine learning, Bayesian inference, and cybersecurity.
\end{IEEEbiographynophoto}

\begin{IEEEbiographynophoto}{George Cybenko}
received his B.Sc. and Ph.D.
degrees in Mathematics from the University of Toronto and Princeton. He
is currently the Dorothy and Walter Gramm Professor
of Engineering at Dartmouth. His research interests include cyber
security, advanced machine learning algorithms and information
deception.
\end{IEEEbiographynophoto}

\begin{IEEEbiographynophoto}{Isaac Dekine}
received the BS and MS degrees in electrical and computer engineering from Carnegie Mellon University in 2006. He is a senior engineer at Systems \& Technology Research LLC. His research interests include RF system design, signal processing, and cybersecurity.
\end{IEEEbiographynophoto}

\begin{IEEEbiographynophoto}{Piyush Kumar}
received the Master of Science degree in Physics from the Indian Institute of Technology Kharagpur in 2001, MS degree in Physics from the University of Chicago in 2004 and the PhD degree in Physics from the University of Michigan Ann Arbor in 2007. He is a lead scientist at Systems \& Technology Research LLC. His research interests include applications of probabilistic methods to problems in physics and engineering, machine learning, and graph theory.
\end{IEEEbiographynophoto}

\begin{IEEEbiographynophoto}{Gil Raz}
received a bachelor's degree in electrical engineering from the Technion -- Israel Institute of Technology in 1988 and a PhD in electrical engineering (minor in mathematics) from the University of Wisconsin -- Madison in 1998. He is chief scientist at Systems \& Technology Research LLC. His research interests include applied mathematics and statistics for solving problems in multiple application areas.
\end{IEEEbiographynophoto}

\end{document}